\def\year{2019}\relax
\documentclass[letterpaper]{article} 
\usepackage{amssymb}
\usepackage{amsmath}
\usepackage{makecell}
\usepackage{booktabs}
\usepackage{xcolor}

\usepackage{aaai19}  
\usepackage{times}  
\usepackage{helvet}  
\usepackage{courier}  
\usepackage{url}  
\usepackage{graphicx}  
\usepackage{stfloats}

\usepackage{amsbsy}
\usepackage{amsmath}
\usepackage{amsmath}%
\usepackage{ntheorem}       
\usepackage{dsfont}
\usepackage{MnSymbol}%
\usepackage{wasysym}%

\makeatletter
\newcommand*{\bdiv}{%
  \nonscript\mskip-\medmuskip\mkern5mu%
  \mathbin{\operator@font div}\penalty900\mkern5mu%
  \nonscript\mskip-\medmuskip
}
\usepackage{wrapfig}
\usepackage[capitalise]{cleveref}
\crefformat{equation}{(#2#1#3)}

\newcommand{\ie}[0]{\emph{i.e.},~}
\newcommand{\eg}[0]{\emph{e.g.},~}

\newcommand{\defeq}{\ensuremath{\doteq}}
\newcommand{\acr}[1]{{{\textsc{#1}}}}

\newcommand{\tuple}[1]{\ensuremath{\left\langle #1 \right\rangle}}
\newcommand{\rel}[1]{\ensuremath{( #1 )}}

\newcommand{\set}[1]{\ensuremath{\mathcal{#1}}}
\newcommand{\func}[1]{\ensuremath{\mathrm{#1}}}
\newcommand{\relset}[0]{{\set{R}~}}
\newcommand{\entset}[0]{{\set{E}~}}
\newcommand{\world}[0]{\set{W}}
\newcommand{\simpleplus}[0]{SimplE$^+$}
\newcommand{\worldc}[0]{\set{W}^\mathsf{c}}
\newcommand{\kg}[0]{\set{K}}
\renewcommand{\Re}[0]{\mathds{R}}
\newcommand{\fim}[0]{\func{Im}}
\newcommand{\freal}[0]{\func{Re}}

\newcommand{\rank}[0]{\func{rank}}

\renewcommand{\psi}{\ensuremath{\mu}}

\newcommand{\vc}[1]{\ensuremath{\mathbf{#1}}}
\newcommand{\xv}[0]{\vc{h}}
\newcommand{\yv}[0]{\vc{t}}
\newcommand{\zv}[0]{\vc{z}}
\newcommand{\rv}[0]{\vc{r}}
\newcommand{\tv}[0]{\vc{t}}
\newcommand{\hv}[0]{\vc{h}}
\newcommand{\ev}[0]{\vc{s}}
\newcommand{\sv}[0]{\vc{s}}

\newcommand{\xe}[0]{h}
\newcommand{\ye}[0]{t}

\newcommand{\re}[0]{r}
\newcommand{\te}[0]{t}
\newcommand{\he}[0]{h}
\newcommand{\ee}[0]{s}

\newcommand{\entityv}[0]{\vc{e}}
\newtheorem{theorem}{Theorem}

\theoremsymbol{\ensuremath{\blacksquare}}
\newtheorem*{proof}{Proof}
\usepackage{natbib}

\begin{document} 
%

\title{Improved Knowledge Graph Embedding using Background Taxonomic Information}
\author{Bahare Fatemi, Siamak Ravanbakhsh, and David Poole\\
The University of British Columbia\\
Vancouver, BC, V6T 1Z4\\
\{bfatemi, siamakx, poole\}@cs.ubc.ca\\
}
\maketitle
\begin{abstract}

Knowledge graphs are used to represent relational information in terms of triples.  To enable learning about domains, embedding models, such as tensor factorization models, can be used to make predictions of new triples. Often there is background taxonomic information (in terms of subclasses and subproperties) that should also be taken into account. We show that existing fully expressive (a.k.a. universal) models cannot provably respect subclass and subproperty information. We show that minimal modifications to an existing knowledge graph completion method enables injection of taxonomic information. Moreover, we prove that our model is fully expressive, assuming a lower-bound on the size of the embeddings. Experimental results on public knowledge graphs show that despite its simplicity our approach is surprisingly effective. 

\end{abstract}

\noindent 
The AI community has long noticed the importance of structure in data. While
traditional machine learning techniques have been mostly focused on feature-based representations, the primary form of data in the subfield of 
Statistical Relational AI (\acr{StaRAI})~\citep{getoor2007introduction, raedt2016statistical} is in the form of entities and relationships among them. Such entity-relationships are often in the form of $( \text{head, relationship, tail})$ triples, which can also be expressed in the form of a graph, with nodes as entities and labeled directed edges as relationships among entities. Predicting the existence, identity, and attributes of entities and their relationships are among the main goals of StaRAI.

Knowledge Graphs (\acr{KG}s) are graph structured knowledge bases that store facts about the world. A large number of \acr{KG}s have been created such as
\acr{NELL} \citep{carlson2010toward}, \acr{Freebase} \citep{bollacker2008freebase}, and Google Knowledge Vault \citep{dong2014knowledge}. These \acr{KG}s have applications in several fields including natural language processing, search, automatic question answering and recommendation systems.
Since accessing and storing all the facts in the world is difficult, \acr{KG}s are incomplete. The goal of \emph{link prediction} for KGs -- a.k.a. \emph{\acr{KG} completion} -- is to predict the unknown links or relationships in a \acr{KG} based on the existing ones. This often amounts to infer (the probability of) new triples from the existing triples.

A common approach to apply machine learning to symbolic data, such as text, graph and entity-relationships, is through embeddings. Word, sentence and paragraph embeddings~\citep{mikolov2013distributed,pennington2014glove},
which vectorize words, sentences and paragraphs using  context information, are widely used in a variety of natural language processing tasks from syntactic parsing to sentiment analysis. Graph embeddings~\citep{hoff2002latent,grover2016node2vec,perozzi2014deepwalk}  are  used in social network analysis for link prediction and community detection. 

In relational learning,  embeddings for entities and relationships are used to generalize from existing data. These embeddings are often formulated in terms of tensor factorization  \citep{nickel2012factorizing,bordes2013translating,trouillon2016complex,kazemi2018simple}. Here, the embeddings are learned such that their interaction through (tensor-)products best predicts the (probability of the) existence of the observed triples; see \citep{nguyen2017overview,wang2017knowledge} for details and discussion. 
Tensor factorization methods have been very successful, yet they rely on a large number of annotated triples to learn useful representations. There is often other information in ontologies which specifies the meaning of the symbols used in a knowledge base. One type of ontological information is represented in a hierarchical structure called a taxonomy. For example, a knowledge base might contain information that DJTrump, whose name is ``Donald Trump'' is a president, but may not contain information that he is a person, a mammal and an animal, because these are implied by taxonomic knowledge. Being told that mammals are chordates, lets us conclude that DJTrump is also a chordate, without needing to have triples specifying this about multiple mammals. We could also have information about subproperties, such as that being president is a subproperty of ``managing'', which in turn is a subproperty of ``interacts with''.

This paper is about combining taxonomic information in the form of subclass and subproperty (\eg managing implies interaction) into relational embedding models. We show that existing factorization models that are fully expressive cannot reflect such constraints for all legal entity embeddings. We propose a model that is provably fully expressive and can represent such taxonomic information, and evaluate its performance on real-world datasets.


\section{Factorization and Embedding}\label{sec:background}
Let \entset represent the set of entities and \relset represent the set of relations.
Let $\world$ be a set of \emph{triples} \rel{h, r, t} that are true in the world, where $h,t \in \entset$ are \emph{head} and \emph{tail}, and $r \in \relset$ is the \emph{relation} in the triple. We use $\worldc$ to represent the triples that are false -- \ie $\worldc \defeq \{ \rel{h,r,t} \in \entset \times \relset \times \entset  \mid \rel{h,r,t} \notin \world \}$. An example of a triple in $\world$ can be $\rel{\text{Paris, CapitalCityOfCountry, France}}$ and an example of a triple in $\worldc$ can be \rel{\text{Paris, CapitalCityOfCountry, Germany}}. A \acr{KG} $\kg \subseteq \world$ is a subset of all the facts. 
The problem of the \acr{KG} completion is to infer $\world$ from its subset \acr{KG}.
There exists a variety of methods for \acr{KG} completion. 
Here, we consider embedding methods and in particular using tensor-factorization. 
For a broader review of the existing \acr{KG} completion that can use background information see Related Work.

\textbf{Embeddings:} An {embedding} is a function from an entity or a relation to a vector (or sometimes higher order tensors) over a field.
We use bold lower-case for vectors -- that is $\ev \in \Re^k$ is an embedding of an entity and
$\rv \in \Re^{l}$ is an embedding of a relation.

\textbf{Taxonomies:}
It is common to have structure over the symbols used in the triples, see \citep[\eg][]{Shoham:2016aa}. The Ontology Web Language (OWL) \citep{Hitzler:2012pd} defines (among many other meta-relations) subproperties and subclasses, where $p_1$ is a subproperty of $p_2$ if $\forall x, y: \rel{x,p_1,y}\rightarrow \rel{x,p_2,y}$, that is whenever $p_1$ is true, $p_2$ is also true. Classes can be defined either as a set with a class assertion (often called ``type") between an entity and a class, \eg saying $x$ is in class $C$ using $\rel{x,type,C}$ or in terms of the characteristic function of the class, a function that is true of element of the class. If $c$ is the characteristic function of class $C$, then $x$ is in class $c$ is written $\rel{x,c,true}$. For representations that treat entities and properties symmetrically, the two ways to define classes are essentially the same. $C_1$ is a subclass of $C_2$ if every entity in class $C_1$ is in class $C_2$, that is,  $\forall x: \rel{x,type,C_1}\rightarrow \rel{x,type,C_2}$ or  $\forall x: \rel{x,c_1,true}\rightarrow \rel{x,c_2,true}$ . If we treat $true$ as an entity, then subclass can be seen as a special case of subproperty. For the rest of the paper we will refer to subsumption in terms of subproperty (and so also of subclass). A non-trivial subsumption is one which is not symmetric; $p_1$ is a subproperty of $p_2$ and there is some relations that is true of $p_1$ that is not true of $p_2$. 
We want the subsumption to be over all possible entities; those entities that have a legal embedding according to the representation used, not just those we know exist. Let \set{E^*} be the set of all possible entities with a legal embedding according to the representation used.


\textbf{Tensor factorization:} For \acr{KG} completion a tensor factorization  defines a 
function $\mu: \Re^{k} \times \Re^{l} \times \Re^{k} \to [0,1]$ that takes the embeddings $\hv$, $\rv$ and $\tv$ of a triple $\rel{h,r,t}$ as input, and generates a prediction, \eg a probability, of the triple being true $\rel{h,r,t} \in \world$. In particular, $\mu$ is often a non-linearity applied to a multi-linear function of $\hv,\rv,\tv$. 
The family of methods that we study
uses the following multi-linear form:
Let $\vc{x}$, $\vc{y}$, and $\vc{z}$ be vectors of length $k$. Define $\tuple{\vc{x},\vc{y},\vc{z}}$ to be the sum of their element-wise product, namely 
\begin{align}\label{eq:prod}
\tuple{\vc{x},\vc{y},\vc{z}} \defeq  \sum_{\ell=1}^{k} \vc{x}_\ell  \vc{y}_\ell \zv_\ell
\end{align}
where $\vc{x}_\ell$ is the $\ell$-th element of vector $\vc{x}$.

Here, we are interested in creating a tensor-factorization method that is fully expressive and can incorporate 
background information in the form of taxonomy. A model is \emph{fully expressive} if given any assignment of truth values to all triples,
there exists an assignment of values to the embeddings of the entities and relations that accurately 
separates the triples belonging to $\world$ and $\worldc$ using $\mu$.


\subsection{ComplEx}\label{sec:complex}
ComplEx \citep{trouillon2016complex} defines the reconstruction function $\mu$, 
such that the embedding of each entity and each relation is a vector of complex numbers. 
Let $\freal(\vc{x})$ and $\fim(\vc{x})$ denote the real and imaginary part of a complex vector \vc{x}.
In ComplEx, the probability of any triple $\rel{h,r,t}$ is 
\begin{align}\label{eq:complex}
\psi(\hv, \rv, \tv) \defeq \sigma(\freal(\tuple{\hv,\rv,\overline{\tv}})
\end{align}
where  $\sigma: \Re \to [0,1]$ is the sigmoid or logistic function, and $\overline{\vc{a}+i\vc{b}} \defeq \vc{a}-i\vc{b}$ (where $i = \sqrt{-1}$) is the element-wise conjugate of the complex vector $\vc{a}+i\vc{b}$. 
Note that, if the tail did not use the conjugate, the head and tail would be treated symmetrically and it could only represent symmetric relations; \eg see DistMult in \cite{yang2014embedding}.


\citet{trouillon2017knowledge} prove that ComplEx is fully expressive. In particular,
they prove that any assignment of ground truth can be modeled by ComplEx embeddings of length $|\entset| |\relset|$.
The following theorem shows that we cannot use ComplEx to enforce our prior knowledge about taxonomies. 

\begin{theorem}
ComplEx cannot enforce non-trivial subsumption.
\end{theorem} 
\begin{proof}
Assume a non-trivial subsumption so that $\forall \xe, \ye \in \set{E^*}: \rel{\xe, \re, \ye} \rightarrow \rel{\xe, \ee, \ye}$, and so $\mu(\xv,\sv,\yv) \geq \mu(\xv,\rv,\yv)$, and there are entities $a, b \in \set{E^*}$ such that $\mu(\vc{a},\sv,\vc{b})>\mu(\vc{a},\rv,\vc{b})$. 
Let $a'$ be an entity such that $\vc{a}'=-\vc{a}$. Then $\mu(\vc{a}',\sv,\vc{b})= 1-\mu(\vc{a},\sv,\vc{b})$ and 
$\mu(\vc{a}',\rv,\vc{b})= 1-\mu(\vc{a},\vc{r},\vc{b})$, so $\mu(\vc{a}',\sv,\vc{b}) < \mu(\vc{a}',\vc{r}, \vc{b})$,
a contradiction to the subsumption we assumed.
$\blacksquare$
\end{proof}

Recently, \cite{ding2018improving} proposed a method which they call ComplEx-NNE+AER to incorporate a weaker notion of subsumption in ComplEx. For a subsumption $\forall \xe, \ye \in \set{E^*}: \rel{\xe,\re,\ye} \rightarrow \rel{\xe,\ee,\ye}$, they suggest adding soft constraints to the loss function to encourage $\freal(\rv) \leq \freal(\sv)$ and $\fim(\rv) = \fim(\sv)$. When the constraints are satisfied, ComplEx-NNE+AER ensures $\forall \xe, \ye \in \entset: \mu(\xv, \rv, \yv) \leq \mu(\xv, \sv, \yv)$. This is a weaker notion than the definition in the Factorization and Embedding section which requires $\forall \xe, \ye \in \entset: \mu(\xv, \rv, \yv) \leq \mu(\xv, \sv, \yv)$ (that is, \set{E^*} is replaced with \set{E}).

\begin{table*}[t]
\caption{Results for the choice of non-linearity in producing non-negative embeddings.}
\label{results-table}
\begin{center}
\begin{tabular}{ccccccccccc}
& \multicolumn{5}{c}{WN18} & \multicolumn{5}{c}{FB15k}                   \\
\cmidrule(lr){2-6} \cmidrule(lr){7-11}
& \multicolumn{2}{c}{MRR} & \multicolumn{3}{c}{Hit@} & \multicolumn{2}{c}{MRR} & \multicolumn{3}{c}{Hit@} \\
\cmidrule(lr){2-3} \cmidrule(lr){4-6} \cmidrule(lr){7-8} \cmidrule(lr){9-11}
Function $f$ & Filter & Raw & 1 & 3 & 10 & Filter & Raw & 1 & 3 & 10 \\ \hline
SimplE$^{+}$-Exponential & $0.866$ & $0.547$ & $0.829$ & $0.925$ & $0.897$ & $0.575$ & $0.248$ & $0.468$ & $0.640$ & $0.773$\\
SimplE$^{+}$-Logistic & $0.854$ & $0.542$ & $0.836$ & $0.863$ & $0.885$ & $0.425$ & $0.228$ & $0.294$ & $0.491$ & $0.694$  \\
SimplE$^{+}$-ReLU & $0.937$ & $0.575$ & $0.936$ & $0.938$ & $0.939$ & $0.725$ & $0.240$ & $0.658$ & $0.770$ & $0.841$ 
\end{tabular}
\end{center}
\end{table*}

\begin{theorem}
ComplEx-NNE+AER cannot satisfy its constraints and be fully expressive if symmetry constraints are allowed.
\end{theorem}

\begin{proof}
In ComplEx a relation $\re$ is symmetric for all possible entities if and only if $\fim(\rv) = 0$ \cite[][Section 3]{trouillon2016complex}.
In order to satisfy constraints for $\forall \xe, \ye \in \set{E}: \rel{\xe,\re,\ye} \rightarrow \rel{\he,\ee,\te}$, \cite{ding2018improving} assign $\fim(\rv) = \fim(\sv)$. Therefore, if relation $\re$ is symmetric, it enforces relation $\ee$ to be symmetric too which is not generally true. As a counter example, $\re$ might be the married\_to relation, which is symmetric (so the $\fim(\textbf{married\_to}) = 0$), but $\ee$ is the knows relation, and $\forall \xe, \ye \in \set{E}: \rel{\xe, married\_to,\ye} \rightarrow \rel{\xe, knows,\ye}$ is true in real-world, but setting the $\fim(\textbf{knows})$ = $\fim(\textbf{married\_to})$ will imply knows is symmetric, which is not true (as many people know celebrities but celebrities do not know many people).
$\blacksquare$
\end{proof}


\subsection{SimplE}\label{sec:simple}
SimplE~\citep{kazemi2018simple} achieves state-of-the-art in \acr{KG} completion by considering two embeddings for each relation: one for the relation $r \in \relset$ itself and one for its inverse.
We use $\rv^+ \in \Re^{k}$ to denote the ``forward'' embedding of $r$ and $\rv^- \in \Re^{k}$ to denote the embedding of its inverse. 
The embedding $\rv = [\rv^+, \rv^-]$ for a relation is a concatenation of these two parts.    
Similarly, the embedding for each entity $e \in \entset$ has two parts: 
its embedding as a head $\entityv^{+}$ and as a tail $\entityv^{-}$ -- that is $\entityv = [\entityv^{+}, \entityv^{-}]$.
Using this notation, SimplE calculates the probability of $\rel{h, r, t} \in \world$ for each triple in both forward and 
backward directions using 
\begin{align}\label{eq:simple}
 \mu(\hv, \rv, \tv) \defeq \sigma \left ( \frac{1}{2} (
 \tuple{\hv^+,\rv^+,\tv^+} + \tuple{\tv^-, \rv^-,\hv^-})
 \right ).
\end{align}

\citet{kazemi2018simple} prove SimplE is fully expressive and provide a bound on the size of the embedding vectors:
For any truth assignment $\world \subseteq \entset \times \relset \times \entset$, there exists a SimplE model with embedding vectors of size $\min(|\entset| |\relset|, |\world| + 1)$ that represent the assignment.
The following theorem shows the limitation of SimplE when it comes to enforcing subsumption.

\begin{theorem}
SimplE cannot enforce non-trivial subsumptions.
\end{theorem}
\begin{proof} 
Consider $\forall \he, \te \in \set{E^*}: \rel{\he, \re, \te} \rightarrow \rel{\he, \ee, \ye}$ as a non-trivial subsumption. So we have $\mu(\xv,\sv,\yv) \geq \mu(\xv,\rv,\yv)$, and there are entities $a, b \in \mathcal{E^*}$ such that $\mu(\vc{a},\sv,\vc{b})>\mu(\vc{a},\rv,\vc{b})$. 
Let $a'$ be an entity such that $\vc{a}'=-\vc{a}$. Then $\mu(\vc{a}',\ev,\vc{b})= 1-\mu(\vc{a},\ev,\vc{b})$ and 
$\mu(\vc{a}',\vc{r},\vc{b})= 1-\mu(\vc{a},\vc{r},\vc{b})$, so $\mu(\vc{a}',\ev,\vc{b}) < \mu(\vc{a}',\vc{r}, \vc{b}))$ a contradiction to the subsumption we assumed.
$\blacksquare$
\end{proof}

\begin{table}
\centering
\caption{Statistics on the datasets.}\label{statistics}
\begin{tabular}{@{} l *5c @{}}
\toprule
\multicolumn{1}{c}{Dataset}    & $|\mathcal{E}|$  & $|\mathcal{R}|$  & \#train  & \#valid & \#test  \\ 
\midrule
 WN18 & 40,943 & 18 & 141,442 & 5,000& 5,000 \\ 
 FB15k & 14,951 & 1,345 & 483,142 & 50,000 & 59,071 \\
 Sport & 1039 & 5 & 1312 & - & 307 \\
 Location & 445 & 5 & 384 & - & 100\\
 \end{tabular}
\end{table}


\begin{table*}
\centering
\footnotesize
\caption{Relations and Rules in Sport and Location datasets.}\label{NELL-rules}
\begin{tabular}{ c|c|c } 
 & Relations & Subsumptions \\
 \hline
 Sport & \makecell{AthleteLedSportsTeam \\ AthletePlaysForTeam \\ CoachesTeam \\ OrganizationHiredPerson \\ PersonBelongsToOrganization} & \makecell{$(x,AtheleLedSportsTeam,y) \rightarrow (x,AthletePlaysForTeam,y)$ \\
 $(x,AthletePlaysForTeam,y) \rightarrow (x,PersonBelongsToOrganization,y)$ \\
 $(x,CoachesTeam,y) \rightarrow (x,PersonBelongsToOrganization,y)$ \\
 $(x,OrganizationHiredPerson,y) \rightarrow (y,PersonBelongsToOrganization,x)$ \\
  $(x,PersonBelongsToOrganization,y) \rightarrow (y,OrganizationHiredPerson,x)$ \\
 } \\ 
 \hline
 \hline
 Location & \makecell{CapitalCityOfCountry \\ CityLocatedInCountry \\ CityLocatedInState \\ StateHasCapital \\ StateLocatedInCountry} & \makecell{$(x,CapitalCityOfCountry,y) \rightarrow (x,CityLocatedInCountry,y)$ \\
$(x,StateHasCapital,y) \rightarrow (y,CityLocatedInState,x)$} \\ 
\end{tabular}
\end{table*}

\subsection{Neural network models}
The neural network models \citep{socher2013reasoning, dong2014knowledge,santoro2017simple} are very flexible, and so without explicit mechanisms to enforce subsumption, they cannot be guaranteed to obey any
subsumption knowledge.

\section{Proposed Variation: SimplE$^{+}$}

In this section we propose a slight modification on SimplE so that the resulting method can enforce subsumption. The modification is restricting entity embeddings to be non-negative -- that is $\entityv^+, \entityv^{-} \geq 0 \; \forall e \in \entset$, where the inequality is element-wise.
Next we show that the resulting model is fully expressive and is able to enforce subsumption.

\begin{theorem}[Expressivity]
For any truth assignment over entities $\entset$ and relations $\relset$ containing $|\world|$ true facts,
there exists a \simpleplus model with embeddings vectors of size $\min(|\entset| |\relset| + 1, |\world|+1)$ that represent the assignment.
\end{theorem}
\begin{proof} 
Assume $r_i$ is the $i$-th relation in \set{R} and $e_j$ is the $j$-th entity in \set{E}. For a vector $a$ we define $(a)_i$ as the $i$-th element of $a$. We define $(\rv^{+}_i)_n = 1$ if $n \bdiv |\set{E}| = i$ except the last element $(\rv^{+}_i)_{|\set{E}| |\set{R}|} = -1$, and for each entity $\ee_i$ we define $(\ev^{+}_{j})_n = 1$ if $n \bmod |\set{E}| = j$ or $n = |\set{E}| |\set{R}|$ and 0 otherwise. In this setting, for each $r_i$ and $e_j$ product of $\rv_i$ and $\ev_i$ is $0$ everywhere except for the element at $(i * |\set{E}| + j)$ and the last element in the embeddings. In order for the triple $\rel{e_j, r_i, e_k}$ to hold, we define $(\ev_k^-)$ to be a vector where all elements are $0$ except the ($i * |\mathcal{E}| + j$)-th element which is $2$.
This proves that \simpleplus is fully expressive with the bound of $|\entset| |\relset| + 1$ for size of the embeddings.

We use induction to prove the bound $|\world|+1$. Let $|\world| = 0$ (base of induction). We can have embedding vectors of size 1 for each entity and relation, setting the value for entities to 1 and to relations to -1. Then \tuple{\hv^+,\rv^+,\tv^+} + \tuple{\tv^-, \rv^-,\hv^-} is negative for every entities $\he$ and $\te$ and relation $\re$. So there exist an assignment of size 1 that represent this ground truth.

Let's assume for any ground truth where $|\world| = n - 1$, there exists an assignment of values to embedding vectors of size $n$ that represent the ground truth (assumption of induction). We must prove for any ground truth where $|\world| = n$, there exist an assignment of values to embedding vectors of size $n + 1$ that represent this ground truth.

Let $\rel{\he, \re, \te}$ be one of the $n$ true facts. Consider a modified ground truth which is identical to the ground truth with $n$ true facts, except that $\rel{\he, \re, \te}$ is assigned false. The modified ground truth has $n - 1$ true facts and based on the assumption of the induction, we can represent it using some embedding vectors of size $n$. Let $q = \tuple{\hv^+,\rv^+,\tv^+} + \tuple{\tv^-, \rv^-,\hv^-}$. We add an element to the end of all embedding vectors and set it to $0$. This increases the vector size to $n + 1$ but does not change any scores. Then we set $\hv$ to $1$, $\rv$ to 1 and $\tv$ to $q + 1$. This ensure this triple is true for the new vectors, and no other probability of triple is affected.
$\blacksquare$
\end{proof}
\begin{theorem}[Subsumption]\label{subsumption-simpleplus}
SimplE$^+$ guarantees subsumption using an inequality constraints. 
\end{theorem}

\begin{proof} 
Assume $\forall \he, \te \in \set{E^*}: \rel{\he, \re, \te} \rightarrow \rel{\he, \ee, \ye}$ as a non-trivial subsumption. As legal entity embeddings in \simpleplus have non-negative elements, by adding the element-wise inequality constraint $\ev \ge \rv$, we force $\mu(\hv, \ev, \tv) \geq \mu(\hv, \rv, \tv)$ for all $\he, \te \in \set{E^*}$ which is forcing the subsumption.
$\blacksquare$
\end{proof}

\subsection{Objective Function and Training}\label{sec:training}
Given the function $\mu$, that maps embeddings to the probability of a triple, ideally we would like to minimize the following regularized negative log-likelihood function:
\begin{align*}
\mathcal{L}(\{\entityv\}, \{\rv\}) =  -\sum_{\rel{h,r,t} \in \world} \log(\mu(\hv, \rv, \tv)) \\ 
- \sum_{\rel{h,r,t} \in \worldc} \log(1 - \mu(\hv, \rv, \tv)) + \Omega(\{\entityv\}, \{\rv\})
\end{align*}
where \{\entityv\} represents entity embeddings, \{\rv\} represents relation embeddings and $\Omega(\{\entityv\}, \{\rv\})$ is a regularization term. We use L2-regularization in our experiments. 
Optimizing $\mathcal{L}$ poses two challenges: \textbf{I}) we do not know the sets $\world$ and $\worldc$,
as the purpose of \acr{KG} completion is to produce these sets in the first place; 
\textbf{II}) the number of triples (specially in $\worldc$) is often too large, 
and for larger \acr{KG}s exact calculation of these terms is often computationally unfeasible.

To address \textbf{I}, we use $\kg$ as a surrogate for $\world$ and use its complement 
$\kg^{\mathsf{c}} = \entset \times \relset \times \entset - \kg$ instead of $\worldc$.
To address the computational problem in \textbf{II}, we use stochastic optimization and follow
the contrastive approach of \cite{bordes2013translating}: for each mini-batch of positive samples from \acr{KG},
we produce a mini-batch of negative samples of the same size, by randomly ``corrupting'' the head or tail of 
the triple -- \ie replacing it with a random entity.

\paragraph{Enforcing the subsumptions} 
In order to enforce $\forall \he, \te \in \set{E^*}: \rel{\he, \re, \te} \rightarrow \rel{\he, \ee, \ye}$, we add an equality constraint as $\rv = \ev - \delta_{\re}$, where $\delta_{r}$ is a non-negative vector that specifies how $\re$ differs from $\ee$. We learn $\delta_{r}$ for all relations $\re$ that are in such a subsumption. This equality constraint guarantees the inequality constraint of Theorem \ref{subsumption-simpleplus}. 



\section{Experimental Results}\label{sec:experiments} 
The objective of our empirical evaluations is two-fold: 
First, we want to see the practical implication of 
non-negativity constraints in terms of effectiveness of training and the quality of final results.
Second, and more importantly, we would like to evaluate the practical benefit of incorporating prior knowledge in the form of subsumptions
in sparse data regimes. 

\textbf{Datasets:}
We conducted experiments on four standard benchmarks: \acr{WN18}, \acr{FB15k}, Sport and Location. \acr{WN18} is a subset of \acr{Wordnet} \citep{miller1995wordnet} and \acr{FB15k} is a subset of \acr{Freebase}~\citep{bollacker2008freebase}. Sport and Location datasets are introduced by \cite{wang2015knowledge}, who created them using NELL \citep{NELL-aaai15}. The relations in Sport and Location, along with the subsumptions, are listed in Table \ref{NELL-rules}. \cref{statistics} gives a summary of these datasets. For evaluation on \acr{WN18}, \acr{FB15k}, we split the existing triples in \acr{KG} into the same train, validation, and test sets using the same split as \citep{bordes2013translating}. 

\textbf{Evaluation Metrics:}\label{sec:metrics}
To evaluate different \acr{KG} completion methods we need to use a train $\set{N}$ and test $\set{T}$ split, where  $\set{N} \cup \set{T} = \kg$. 
We use two evaluation metrics: \acr{hit@t} and Mean Reciprocal Rank (\acr{MRR}). 
Both these measures rely on the \emph{ranking} of a triple in the test set $\rel{h,r,t} \in \set{T}$, obtained by corrupting the head 
(or the tail) of the relation with $h' \neq h$ and estimating $\mu\rel{h',r,t}$. An indicator for a good \acr{KG} completion method is that 
$\rel{h,r,t}$ ranks high in the sorted list among corrupted triples. 

Let $\rank_h\rel{h,r,t}$ be the ranking of $\mu\rel{h,r,t}$ among all head-corrupted relations, and 
let $\rank_t\rel{h,r,t}$ denote a similar ranking with tail corruptions.
\acr{MRR} is the mean of the reciprocal rank:
\begin{align*} 
\func{MRR} \defeq \frac{1}{2 * |\set{T}|} \sum_{\rel{h, r, t} \in \set{T}} \frac{1}{\rank_h\rel{h,r,t}} + \frac{1}{\rank_t\rel{h,r,t}}
\end{align*}
To provide a better metric, \citet{bordes2013translating} suggest removing any corrupted relation that is in \acr{KG}. We refer to the original definition of \acr{MRR} as raw \acr{MRR} and to \citet{bordes2013translating}'s modified version as filtered \acr{MRR}.

\acr{hit@t} measures the proportion of triples in $\set{T}$ that rank among top $t$ after corrupting both heads and tails. 

\subsection{Effect of Non-Negativity Constraints}\label{sec:nonneg}
Non-negativity has been a subject studied in various research fields. In many NLP-related tasks, non-negativity constraints are studies to learn more interpretable representations for words \citep{murphy2012learning}. In matrix factorization, non-negativity constraints are used to produce more coherent and independent factors~\citep{lee1999learning}. \cite{ding2018improving} also proposed using non-negativity constraint to incorporate subsumption into ComplEx. 
We use the non-negativity constraint in SimplE$^+$ to enforce monotonousity of probabilities as dictated by subsumption.
In order to get non-negativity constraint on the embedding of entities, we simply apply an element-wise non-linearity $\phi: \Re \to \Re^{\geq 0}$ before evaluation -- that is 
we replace $\mu(\hv,\rv,\tv)$ with $\mu(\phi(\hv), \rv, \phi(\tv))$. 

\cref{results-table} shows the result of SimplE$^{+}$ with for different choices of $\phi$: 
\textbf{I}) exponential $\phi(x) = e^x$; 
\textbf{II}) logistic $\phi(x) = {(1 + e^{-x})}^{-1}$; and \textbf{III}) rectified linear unit (ReLU) $\phi(x) = \max (x, 0)$. 
ReLU outperforms other choices, and therefore moving forward we use ReLU for non-negativity constraints.

Next, we evaluate the effect of non-negativity constraint on the performance of the algorithm. 
\cref{results-table-pos} shows our result on \acr{WN18} and \acr{FB15k} datasets. Note that this is effectively comparing \simpleplus with SimplE and ComplEx, without accommodating any subsumptions.
As the results indicate, this constraint does not deteriorate the model's performance.





\begin{table*}[t]
\caption{Results on \acr{WN18} and \acr{fb15k} for SimplE and SimplE$^+$ without incorporating subsumptions.}
\label{results-table-pos}
\begin{center}
\begin{tabular}{ccccccccccc}
& \multicolumn{5}{c}{WN18} & \multicolumn{5}{c}{FB15K}                   \\
\cmidrule(lr){2-6} \cmidrule(lr){7-11}
& \multicolumn{2}{c}{MRR} & \multicolumn{3}{c}{Hit@} & \multicolumn{2}{c}{MRR} & \multicolumn{3}{c}{Hit@} \\
\cmidrule(lr){2-3} \cmidrule(lr){4-6} \cmidrule(lr){7-8} \cmidrule(lr){9-11}
Model& Filter & Raw & 1 & 3 & 10 & Filter & Raw & 1 & 3 & 10 \\ \hline
ComplEx & $0.941$ & $0.587$ & $0.936$ & $0.945$ & $0.947$ & $0.692$ & $0.242$ & $0.599$ & $0.759$ & $0.840$ \\
SimplE & $0.942$ & $0.588$ & $0.939$ & $0.944$ & $0.947$ & $0.727$ & $0.239$ & $0.660$ & $0.773$ & $0.838$ \\
SimplE$^{+}$ & $0.937$ & $0.575$ & $0.936$ & $0.938$ & $0.939$ & $0.725$ & $0.240$ & $0.658$ & $0.770$ & $0.841$ \\
\end{tabular}
\end{center}
\end{table*}

\subsection{Sparse Relations}
In this section, we study the scenario of learning relations that appear in few triples in the \acr{KG}. 
In particular, we observe the behaviour of various methods as the amount of training triples varies. 
We train SimplE, \simpleplus, and logical inference on fractions of the Sport training set and test them on the full test set. Logical inference refers to inferring new triples based only on the subsumptions.

Figure \ref{fig:training-data} shows the \acr{hit@1} of the three methods when they are trained on different fractions (percentages) of the training data. According to Figure \ref{fig:training-data}, when training data is scarce, logical inference performs better than (or on-par with) SimplE, as SimplE does not see enough triples to be able to learn meaningful embeddings. As the amount of training data increases, SimplE starts to outperform logical inference as it can better generalize to unseen cases than pure logical inference. The gap between these two methods becomes larger as the amount of training data increases. For all tested fractions, \simpleplus outperforms both SimplE and logical inference as it uses both the generalization power of SimplE and the inference power of logical rules.

\begin{figure}[t]
\includegraphics[trim=20pt 0 20pt 0pt, width=8cm]{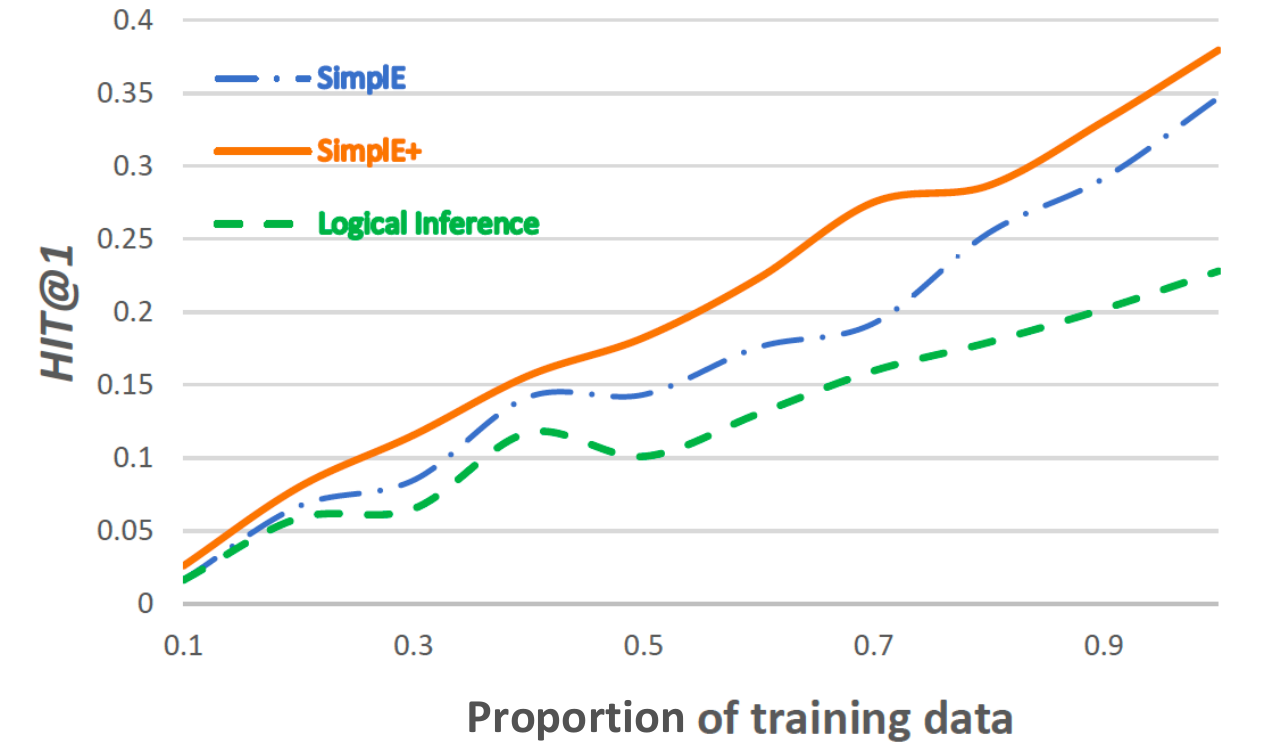}
\centering
\caption{\emph{hit@1 of SimplE, \simpleplus, and logical inference for different proportions of training data on Sport dataset}}
\label{fig:training-data}
\end{figure}

In order to test the effect of incorporating taxonomical information on the number of epochs required for training to converge, we tested SimplE and \simpleplus on the Sport dataset with the same set of parameters and the same initialization and plotted the loss function for each epoch. The plot in Figure~\ref{fig:convergence} shows that \simpleplus requires fewer epochs than SimplE to converge. 
\subsection{\acr{KG}s with no Redundant Triples}
Tensor factorization techniques rely on large amounts of annotated data. When background knowledge is available, we might expect a \acr{KG} to not include redundant information. For instance if we have $\rel{Paris, CapitalCityOfCountry, France}$ in a \acr{kg} and we know $\forall \he, \te \in \set{E^*}: \rel{\he, CapitalCityOfCountry, \te} \rightarrow \rel{\he, CityLocatedInCountry, \te} $, then the triple $\rel{Paris, CityLocatedInCountry, France}$ is redundant. Similar to the experiment for incorporating background knowledge in \cite{kazemi2018simple}, we remove all redundant triples from the training set and compare SimplE with \simpleplus and logical inference. The obtained results in Table~\ref{results-table-inject} demonstrate that \simpleplus outperforms SimplE and logical inference on both Sport and Location datasets with a large margin. As an example, \simpleplus gains almost 90 percent and 230 percent improvement over SimplE in terms of \acr{hit@1} for Sport and Location datasets respectively. These results represent the clear advantage of \simpleplus over SimplE when background taxonomic information is available.  


\begin{figure}[t]
\includegraphics[width=9cm]{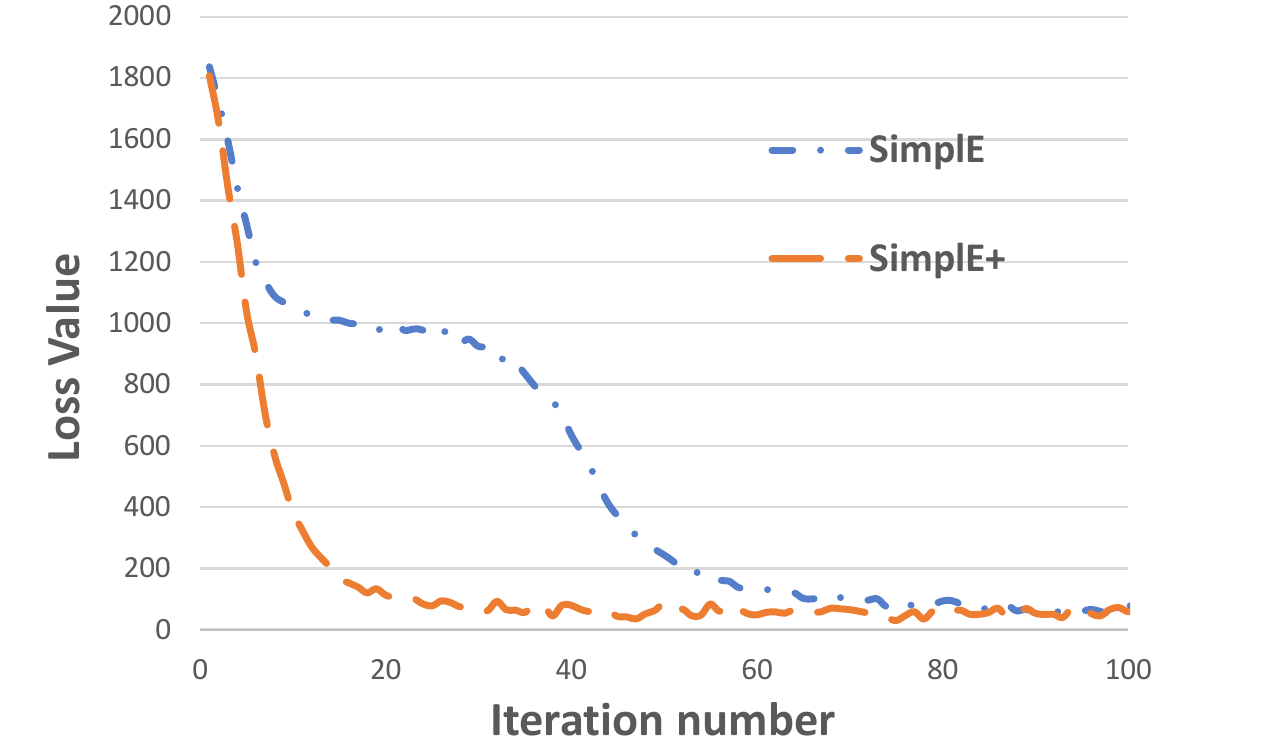}
\centering
\caption{\emph{Loss value at each epoch for SimplE and \simpleplus on Sport dataset.}}
\label{fig:convergence}
\end{figure}

\section{Related Work} \label{related-works}
Incorporating background knowledge in link prediction methods has been the focus of several studies. 
Here, we categorize these approaches emphasizing the shortcomings that are addressed in our work; see \citep{nickel2016review} for a review of \acr{KG} embedding methods. 


\begin{table*}[bp]
\caption{Results on Sport and Location. Best results are in bold. MRR and Hit@n for n $>$ 1 does not make sense for logical inference.}
\label{results-table-inject}
\begin{center}
\begin{tabular}{ccccccccccc}
& \multicolumn{5}{c}{Sport} & \multicolumn{5}{c}{Location}                   \\
\cmidrule(lr){2-6} \cmidrule(lr){7-11}
& \multicolumn{2}{c}{MRR} & \multicolumn{3}{c}{Hit@} & \multicolumn{2}{c}{MRR} & \multicolumn{3}{c}{Hit@} \\
\cmidrule(lr){2-3} \cmidrule(lr){4-6} \cmidrule(lr){7-8} \cmidrule(lr){9-11}
Model& Filter & Raw & 1 & 3 & 10 & Filter & Raw & 1 & 3 & 10 \\ \hline
Logical inference & - & - & $0.288$ & - & - & - & - & $0.270$ & - & -\\
SimplE & $0.230$ & $0.174$ & $0.184$ & $0.234$ & $0.324$ & $0.190$ & $0.189$ & $0.130$ & $0.210$ & $0.315$ \\
SimplE$^{+}$ & $\textbf{0.404}$ & $\textbf{0.337}$ & $\textbf{0.349}$ & $\textbf{0.440}$ & $\textbf{0.508}$ & $\textbf{0.440}$ & $\textbf{0.434}$ & $\textbf{0.430}$ & $\textbf{0.440}$ & $\textbf{0.450}$
\end{tabular}
\end{center}
\end{table*}

\textbf{Soft rules} There is a large family of link prediction models based on soft first-order logic rules \cite{richardson2006markov,de2007problog,kazemi2014relational}. While these models can be easily integrated with background taxonomic information, they typically cannot generalize to unseen cases beyond their rules. Exceptions include \cite{fatemi2016learning,kazemi2018relnn} which combine (stacked layers of) soft rules with entity embeddings, but these models have only applied to property prediction. Approaches based on path-constrained random walks (e.g., \cite{lao2010relational}) suffer from similar limitations as they have been shown to be a subset of probabilistic logic-based models \cite{kazemi2018bridging}.

\textbf{Augmentation by grounding of the rules} The simplest way to incorporate a set of rules in the \acr{KG} is to augment the \acr{KG} with their groundings~\citep{sedghi2018knowledge} \emph{before} learning the embedding. 
\citet{demeester2016lifted} address the computational inefficiency of this approach through {lifted} rule injection.
However, in addition to being inefficient, the the resulting model does not guarantee the subsumption in the completed \acr{KG}.

\textbf{Augmentation through post-processing}
A simple approach is to augment the \acr{KG} \emph{after} learning the embedding using an existing method~\citep{wang2015knowledge, wei2015large}. 
That is, as a post processing step we can modify the output of \acr{KG} completion so as to satisfy the ontological constraints. 
The drawback of this approach is that the background knowledge does not help learn a better representation. 

\textbf{Regularized embeddings}
\citet{rocktaschel2015injecting} regularize the learned embeddings using first-order logic rules. In this work, every logic rule is grounded based on observations and a differentiable term is added to the loss function for every grounding. For example, grounding the rule $\forall x: human(x) \rightarrow animal(x)$ would result in a very large number of loss terms to be added to the loss function in a large \acr{KG}. 
This method as well as other approaches in this category \citep[\eg][]{rocktaschel2014low,wang2015knowledge,wang2016learning} do not scale beyond a few entities and rules, because of the very large number of regularization terms added to the loss function~\citep{demeester2016lifted}. \citet{guo2017knowledge} proposed a methods for incorporating entailment into ComplEx called RUGE which models rules based on t-norm fuzzy logic, which imposes an independence assumption over the atoms. Such an independence assumption is not necessarily true, especially in the case of subsumption, e.g. in $human(x) \rightarrow animal(x)$ for which the left and the right part of the subsumption are strongly dependent. In addition to being inefficient, the resulting model of the regularized embedding approaches does not guarantee the subsumption in the completed \acr{KG}.

\textbf{Constrained matrix factorization}
Several recent works incorporate background ontologies into the embeddings learned by matrix factorization~\citep[\eg][]{rocktaschel2015injecting,demeester2016lifted}.
While these methods address the problems of the two categories above, they are inadequate due to the use of matrix factorization.
Application of matrix factorization for \acr{KG} completion~\citep{riedel2013relation} learns a distinct embedding for each head-tail combination.
In addition to its prohibitive memory requirement, since entities do not have their own embeddings, some regularities in the \acr{KG} are 
ignored; for example this representation is oblivious to the fact that $(h_i, r_k, t_j)$ and $(h_l, r_m, t_j)$ share the same tail.

\textbf{Constrained translation-based methods}
In translation-based methods, the relation between two entities is represented using an affine transformation, often in the form of translation.
Most relevant to our work is \acr{KALE} \citep{guo2016jointly} that constrains the representation to accommodate logical rules, albeit after costly propositionalization.
Several recent works show that a variety of existing translation-based methods are not fully expressive~\citep{wang2017knowledge,kazemi2018simple}, 
putting a severe limitation on the kinds of \acr{KG}s that can be modeled using translation-based approaches. 

\textbf{Region based representation}
\citet{gutierrez2018knowledge} propose representing relations as convex regions in a $2k$-dimensional space, where $k$ is the length of the entity embeddings. A relation between two embeddings is deemed true if the corresponding point is in the convex region of the relation. Although this framework 
allows \citet{gutierrez2018knowledge} to incorporate a subset of existential rules by restricting the convex regions of relations, they did not propose a practical method for learning and their method is restricted to a subset of existential rules.


\section{Conclusion and Future Work}
In this paper, we proposed \simpleplus, a fully expressive tensor factorization model for knowledge graph completion when background taxonomic information (in terms of subclasses and subproperties) is available. We showed that existing fully expressive models cannot provably respect subclass and subproperty information. Then we proved that by adding non-negativity constraints to entity embeddings of SimplE, a state-of-the-art tensor factorization approach, we can build a model that is not only fully expressive but also able to enforce subsumptions. Experimental results on benchmark \acr{KG}s demonstrate that \simpleplus is simple yet effective. On our benchmarks, \simpleplus outperforms SimplE and offers a faster convergence rate when background taxonomic information is available. In future, we plan to extend \simpleplus  to further incorporate ontological background information, and rules such as $\forall \xe, \ye \in \set{E^*}: \rel{\xe, \re, \ye} \land \rel{\xe, \ee, \ye} \rightarrow \rel{\xe, p, \ye}$. 
\bibliography{aaai}
\bibliographystyle{aaai}
\end{document}